\documentclass[11pt]{article}

\usepackage[margin=1in]{geometry}
\usepackage{algorithm}
\usepackage{algorithmic}
\usepackage{amsmath, amssymb, amsthm}
\usepackage{hyperref}
\usepackage[nameinlink,capitalize,noabbrev]{cleveref}
\usepackage{graphicx}
\usepackage{natbib}
\usepackage{url}            
\usepackage{booktabs}       
\usepackage{amsfonts}       
\usepackage{nicefrac}       
\usepackage{microtype}      
\usepackage{xcolor}         
\usepackage{mathtools}
\usepackage{algorithm}
\usepackage{dsfont}
\usepackage{bbm}
\usepackage{longtable}
\usepackage{multirow}
\usepackage{tikz}
\usepackage{subcaption}
\usetikzlibrary{arrows.meta, positioning, shapes.geometric, backgrounds}
\usepackage[affil-it]{authblk}

\theoremstyle{plain}
\newtheorem{theorem}{Theorem}[section]
\newtheorem{proposition}[theorem]{Proposition}

\theoremstyle{definition}

\theoremstyle{remark}

\title{Reducing Learner Redundancy in Boosting via Residual Orthogonalization}

\author[a]{Ye Su}
\author[b]{Jipeng Guo}
\author[c]{Xin Xu}
\author[a]{Gangchun Zhang}
\author[a]{Jinxin Chen}
\author[d,*]{Di Wu}
\author[a,*]{Longlong Zhao}

\affil[a]{Shenzhen Institutes of Advanced Technology, Chinese Academy of Sciences, Shenzhen 518055, China}
\affil[b]{College of Information Science and Technology, Beijing University of Chemical Technology, Beijing 100029, China}
\affil[c]{School of Computer Science, Central China Normal University, Hubei 430000, China}
\affil[d]{the School of Computing, Engineering and Mathematical Sciences, La Trobe University, Melbourne VIC 3086, Australia}

\affil[*]{Corresponding authors. Emails: d.wu@latrobe.edu.au, ll.zhao@siat.ac.cn}

\date{}

\begin{document}

\maketitle

\begin{abstract}
While sequential residual fitting is the bedrock of standard boosting frameworks, it inherently breeds learner redundancy by repeatedly revisiting correlated error components. To address this bottleneck, we propose a shift from residual fitting to \textit{residual orthogonalization} and introduce SCBoost. Our framework tackles redundancy through two complementary mechanisms: Spectral Residual Projection (SRP) and Covariance-Regularized Weighting (CRW). During training, SRP projects each residual target onto the orthogonal complement of the historical prediction subspace, forcing successive learners to capture only novel empirical innovations. During aggregation, CRW optimizes ensemble weights on a validation set with an explicit covariance penalty to mitigate remaining correlations. Theoretically, we provide a finite-sample geometric characterization proving that SRP yields an exact additive residual-energy decomposition. Furthermore, under an isotropic-noise assumption, we rigorously establish the conditions under which this projection improves the effective Signal-to-Noise Ratio. Extensive experiments across ten benchmark datasets demonstrate that SCBoost delivers strong out-of-the-box performance, particularly in accuracy and F1 score. This work reinterprets boosting through a geometric lens, suggesting that explicit redundancy control is a principled and necessary step toward more efficient ensemble architectures.
\end{abstract}


\begin{figure*}[t]
    \centering
    \begin{subfigure}[b]{0.48\textwidth}
        \centering
        \begin{tikzpicture}[scale=0.8]
            \coordinate (O) at (0,0);
            \coordinate (T) at (5,4); 

            \fill[gray!10] (-0.5,-0.5) rectangle (5.5,4.5);
            \node[gray!50] at (5, 0.5) {Full Function Space};

            \draw[->, red!80, line width=1.5pt, line cap=round] (O) -- (2.5, 1.0) node[midway, below left, black, scale=0.8] {$h^{(1)}$};
            \draw[->, red!80, line width=1.5pt, line cap=round] (2.5, 1.0) -- (3.2, 2.5) node[midway, right, black, scale=0.8] {$h^{(2)}$};
            \draw[->, red!80, line width=1.5pt, line cap=round] (3.2, 2.5) -- (4.2, 1.8) node[midway, above, black, scale=0.8] {$h^{(3)}$};
            \draw[->, red!80, line width=1.5pt, line cap=round] (4.2, 1.8) -- (4.8, 3.5) node[midway, right, black, scale=0.8] {$h^{(4)}$};

            \draw[dashed, gray] (2.5, 1.0) circle (0.5);
            \draw[dashed, gray] (4.2, 1.8) circle (0.4);
            \node[gray, scale=0.7, align=center] at (1.5, 2.5) {Redundant \\ Gradient Components};

            \filldraw[black] (T) circle (2pt) node[above right] {Optimal $F^*$};
            
            \node[below, font=\bfseries] at (2.5,-1) {(a) Standard Residual Fitting};
        \end{tikzpicture}
    \end{subfigure}
    \hfill 
    \begin{subfigure}[b]{0.48\textwidth}
        \centering
        \begin{tikzpicture}[scale=0.8]
            \coordinate (O) at (0,0);
            \coordinate (T) at (5,4);

            \fill[blue!5] (-0.5,-0.5) rectangle (5.5,4.5);

            \draw[->, blue!80, line width=1.5pt, line cap=round] (O) -- (3.5, 0) node[midway, below, black, scale=0.8] {$h^{(1)}$};
 
            \draw[dashed, blue!40] (0,0) -- (4,0);
            \node[blue!60, scale=0.7] at (4.5, 0.3) {$\mathcal{H}_{t-1}$};
 
            \draw[->, blue!80, line width=1.5pt, line cap=round] (3.5, 0) -- (3.5, 3.0) node[midway, right, black, scale=0.8] {$h^{(2)} \perp h^{(1)}$};

            \draw[blue!60] (3.5, 0.3) -- (3.2, 0.3) -- (3.2, 0);

            \draw[->, blue!80, line width=1.5pt, line cap=round] (3.5, 3.0) -- (4.8, 3.8) node[midway, above left, black, scale=0.8] {$h^{(3)}$};

            \node[blue!80, scale=0.7, align=center] at (1.5, 3.5) {Purified Target \\ (Orthogonal Innovation)};

            \filldraw[black] (T) circle (2pt) node[above right] {Optimal $F^*$};

            \node[below, font=\bfseries] at (2.5,-1) {(b) SCBoost (Ours)};
        \end{tikzpicture}
    \end{subfigure}
    
    \caption{\textbf{Conceptual Landscape of Boosting Paradigms.} (a) Standard Boosting performs greedy descent in the original function space, leading to redundant updates and ``zig-zagging" behavior. (b) SCBoost projects each residual onto the orthogonal complement of the historical subspace $\mathcal{H}_{t-1}$ \textit{before} training, ensuring that each new learner captures geometrically distinct information and accelerates convergence towards the optimal ensemble $F^*$.}
    \label{fig:paradigm_shift}
\end{figure*}
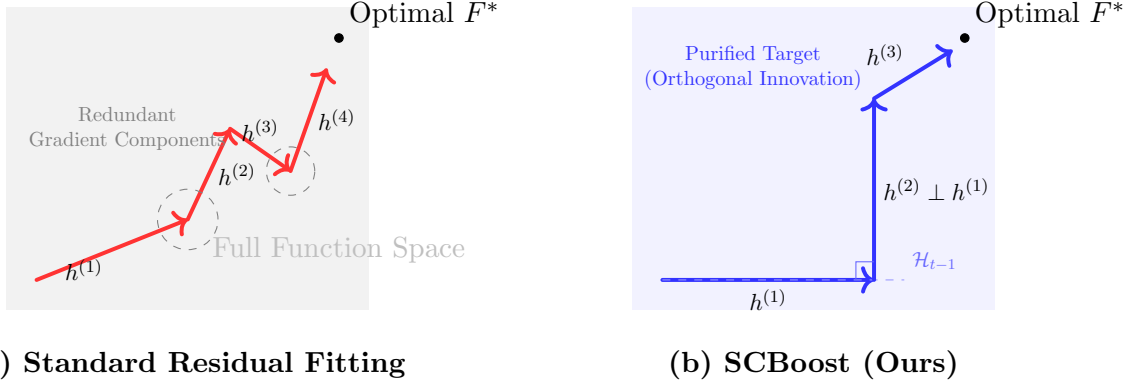

\section{Introduction}

Modern boosting frameworks like XGBoost and LightGBM \cite{chen2016xgboost, ke2017lightgbm} owe their success primarily to engineering optimizations in efficiency and feature handling. However, their statistical core, sequential residual fitting, has remained largely unchanged for decades \cite{schapire1999brief, buhlmann2007boosting, mayr2014evolution}. This paradigm inherently breeds redundancy by training new learners on highly correlated residuals, creating a ``redundancy bottleneck'' that limits the ensemble's generalization capability. Breaking this ceiling requires addressing this statistical limitation directly, even if it necessitates rethinking the computational trade-offs that have defined the past decade of development.

While prior efforts have sought to promote diversity through methods like negative correlation learning (NCL) \cite{liu1999ensemble, liu2000evolutionary, wang2010negative} or randomization \cite{dietterich2000experimental, kotsiantis2011combining, wen2020batchensemble}, they often treat diversity as a secondary objective, imposing soft penalties that create an unstable trade-off with fitting accuracy. More importantly, they fail to address the root cause of the problem in boosting, the correlated nature of the residual targets themselves \cite{bartlett1998boosting}.

To break the redundancy bottleneck, we propose a shift from residual fitting to residual orthogonalization. As illustrated in Figure~\ref{fig:paradigm_shift}, this mechanism compels each new learner to approximate only the error component geometrically distinct from the existing ensemble. By projecting the learning target onto the functional null space of historical predictors, we eliminate signal redundancy and promote the extraction of novel information. The primary contribution of this study is the introduction of \textbf{residual orthogonalization} as a fundamental paradigm shift for boosting. We move beyond the traditional fit-and-add logic and establish a new implementation, \textbf{SCBoost}, whose contributions are structured as follows:
\begin{itemize}
\item \textbf{Core Principle}. We introduce \textbf{Spectral Residual Projection (SRP)} to modify the residual target before fitting each new learner. By applying spectral decomposition to the prediction history, SRP projects the residual onto the orthogonal complement of the selected historical prediction subspace.

\item \textbf{Geometric Characterization}. We show that SRP is an exact empirical orthogonal projection on the training sample. This yields an additive decomposition of the residual energy into a historical component and an orthogonal innovation component.

\item \textbf{Noise Interpretation}. Under an explicit fixed-subspace isotropic-noise assumption, we characterize how projection changes the noise energy and the effective SNR. The analysis shows that SNR improves only when the removed signal fraction is smaller than the removed noise fraction.

\item \textbf{Aggregation Strategy}. We introduce \textbf{Covariance-Regularized Weighting (CRW)} to aggregate the learned predictors. CRW uses validation-set covariance regularization to reduce the influence of highly correlated learners, and is motivated by the ambiguity decomposition under squared loss.
\end{itemize}

\section{SCBoost}
Standard boosting fits residuals sequentially but does not explicitly control the correlation between new learners and the existing ensemble, which can lead to redundant updates (Figure \ref{fig:paradigm_shift}(a)). We propose SCBoost, a boosting framework based on residual orthogonalization. Instead of fitting the raw residual directly, Spectral Residual Projection (SRP) projects the residual target onto the orthogonal complement of the historical prediction subspace (Figure \ref{fig:paradigm_shift}(b)). To aggregate the learned predictors, Covariance-Regularized Weighting (CRW) then assigns ensemble weights with an additional covariance penalty. SRP controls the target used for learner induction, while CRW controls the final aggregation.

\subsection{Spectral Residual Projection}
\label{subsec:SRP}

\textbf{Algorithm Description.}
Let the residual vector at iteration $t$ be $\mathbf{r}^{(t)} = \mathbf{y} - \sigma(\mathbf{F}^{(t-1)})$, where $\mathbf{F}^{(t-1)}$ is the current logit output on the training data and $\sigma(\cdot)$ is the sigmoid function. We maintain a prediction history matrix $\mathbf{H}^{(t-1)} = [\mathbf{h}^{(1)},\mathbf{h}^{(2)},\dots,\mathbf{h}^{(t-1)}] \in \mathbb{R}^{n\times(t-1)}$, where $\mathbf{h}^{(j)}$ denotes the prediction vector of the $j$-th base learner on the training data.

At step $t$, we compute the singular value decomposition $\mathbf{H}^{(t-1)} = \mathbf{U}\mathbf{\Sigma}\mathbf{V}^{\top}$. Let $\mathbf{U}_k=[\mathbf{u}*1,\dots,\mathbf{u}*k]\in\mathbb{R}^{n\times k}$ be the top-$k$ left singular vectors selected by the energy threshold $\alpha\in(0,1)$:
\begin{equation*}
\frac{\sum*{i=1}^{k}\sigma_i^2} {\sum*{j=1}^{\min(n,t-1)}\sigma_j^2} \geq \alpha .
\end{equation*}
We define the projection operators $\mathbf{P}_k=\mathbf{U}_k\mathbf{U}_k^\top$, $\mathbf{Q}_k=\mathbf{I}_n-\mathbf{P}_k$. The projected residual target is $\tilde{\mathbf{r}}^{(t)} = \mathbf{Q}_k\mathbf{r}^{(t)} = \mathbf{r}^{(t)} - \mathbf{U}_k(\mathbf{U}_k^\top\mathbf{r}^{(t)})$. The next learner is trained to fit $\tilde{\mathbf{r}}^{(t)}$ rather than the raw residual $\mathbf{r}^{(t)}$. This projection guarantees orthogonality of the training target to the selected historical prediction subspace. It does not by itself guarantee that the fitted learner $\mathbf{h}^{(t)}$ is exactly orthogonal to the previous learners, since the fitted learner also depends on the approximation capacity and optimization procedure of the base learner.

\textbf{Theoretical Analysis.}
We first record the deterministic geometric property of SRP. This result is finite-sample and does not require a distributional assumption.

\begin{proposition}[Empirical Orthogonal Projection. The detialed proof in Appendix \ref{app:sec:a}]
\label{prop:ortho}
Let $\mathcal{H}_{t-1}=\operatorname{span}(\mathbf{U}_k)$, where $\mathbf{U}_k$ has orthonormal columns. Let $\mathbf{P}_k=\mathbf{U}_k\mathbf{U}_k^\top$ and $\mathbf{Q}_k=\mathbf{I}_n-\mathbf{P}_k$. For any residual vector $\mathbf{r}^{(t)}\in\mathbb{R}^n$, the SRP target $\tilde{\mathbf{r}}^{(t)}=\mathbf{Q}*k\mathbf{r}^{(t)}$ satisfies $\tilde{\mathbf{r}}^{(t)} = \mathop{\arg\min}*{\mathbf{z}\in\mathbb{R}^n} \quad |\mathbf{z}-\mathbf{r}^{(t)}|*2^2 \ \text{s.t.} \quad \langle \mathbf{z},\mathbf{u}\rangle=0, \quad \forall \mathbf{u}\in\mathcal{H}*{t-1}.$ Moreover, $|\mathbf{r}^{(t)}|_2^2 = |\mathbf{P}_k\mathbf{r}^{(t)}|_2^2 + |\mathbf{Q}_k\mathbf{r}^{(t)}|_2^2$. Equivalently,
\begin{equation*}
|\tilde{\mathbf{r}}^{(t)}|_2^2 = |\mathbf{r}^{(t)}|*2^2 - \sum*{i=1}^{k}(\mathbf{u}_i^\top\mathbf{r}^{(t)})^2 .
\end{equation*}
\end{proposition}

Proposition~\ref{prop:ortho} shows that SRP removes exactly the component of the current residual lying in the selected historical prediction subspace. The result is an empirical statement on the training vectors. It should not be interpreted as a guarantee of functional orthogonality outside the training sample.

We next give a limited statistical interpretation of SRP under an explicit fixed-subspace noise model. The assumption that the projection is fixed relative to the noise is essential. In the standard boosting implementation, the historical learners are trained from the same labels, so the projection matrix can be label-noise dependent. The following result should therefore be read as a fixed-subspace characterization of the projection operation, not as an unconditional robustness theorem for the full adaptive algorithm.

\begin{proposition}[Fixed-Subspace Noise and SNR Characterization. The detialed proof in Appendix \ref{app:sec:b}]
\label{prop:noise_snr}
Let $\mathbf{r}=\mathbf{s}+\boldsymbol{\epsilon}$, where $\mathbf{s}\in\mathbb{R}^n$ is deterministic and $\boldsymbol{\epsilon}\sim\mathcal{N}(\mathbf{0},\nu^2\mathbf{I}_n)$. Assume that $\mathbf{P}_k$ is a fixed rank-$k$ orthogonal projector independent of $\boldsymbol{\epsilon}$, and let $\mathbf{Q}_k=\mathbf{I}_n-\mathbf{P}_k$ with $d=\operatorname{rank}(\mathbf{Q}_k)=n-k$. Then
\begin{equation*}
\mathbb{E}|\mathbf{Q}_k\boldsymbol{\epsilon}|_2^2 = d\nu^2 = \left(1-\frac{k}{n}\right) \mathbb{E}|\boldsymbol{\epsilon}|_2^2 .
\end{equation*}
Furthermore, for any $\delta\in(0,1)$, with probability at least $1-\delta$,
\begin{equation*}
|\mathbf{Q}_k\boldsymbol{\epsilon}|_2^2 \leq \nu^2 \left[ d + 2\sqrt{d\log(1/\delta)} + 2\log(1/\delta) \right].
\end{equation*}
For the full projected residual, for any $\eta>0$,
\begin{equation*}
|\mathbf{Q}_k\mathbf{r}|_2^2 \leq (1+\eta)|\mathbf{Q}_k\mathbf{s}|_2^2 + (1+\eta^{-1})|\mathbf{Q}_k\boldsymbol{\epsilon}|_2^2 .
\end{equation*}
If $0\leq k<n$ and $\mathbf{s}\neq\mathbf{0}$, define $\operatorname{SNR}(\mathbf{r}) = \frac{|\mathbf{s}|_2^2}{n\nu^2}$ and $\operatorname{SNR}(\mathbf{Q}_k\mathbf{r}) = \frac{|\mathbf{Q}_k\mathbf{s}|_2^2}{(n-k)\nu^2}$. Let $\rho = \frac{|\mathbf{P}_k\mathbf{s}|_2^2}{|\mathbf{s}|_2^2}$ be the fraction of signal energy removed by the projection. Then
\begin{equation*}
\frac{\operatorname{SNR}(\mathbf{Q}_k\mathbf{r})} {\operatorname{SNR}(\mathbf{r})} = \frac{1-\rho}{1-k/n}.
\end{equation*}
Thus, the SNR improves if and only if $\rho < \frac{k}{n}$.
\end{proposition}

Proposition~\ref{prop:noise_snr} states the condition under which projection improves the effective SNR. SRP always reduces the expected energy of isotropic noise under the fixed-subspace assumption, but it improves SNR only when the removed signal fraction is smaller than the removed noise fraction. This condition is important because an overly aggressive projection may also remove useful residual signal.

The optimization effect of SRP can also be described by a simple one-step identity. Let $\mathbf{r} = \mathbf{P}_k\mathbf{r} + \mathbf{Q}_k\mathbf{r} = \mathbf{p}+\mathbf{q}$. Assume that the learner prediction vector $\mathbf{h}$ is normalized, $|\mathbf{h}|_2=1$, and lies in the projected subspace, i.e., $\mathbf{h}\in\operatorname{Range}(\mathbf{Q}_k)$. With the optimal scalar step size $\eta^\star=\langle \mathbf{q},\mathbf{h}\rangle$, we have
\begin{equation*}
\begin{aligned}
|\mathbf{r}-\eta^\star\mathbf{h}|_2^2 &= |\mathbf{p}|_2^2 + |\mathbf{q}|_2^2 - \langle \mathbf{q},\mathbf{h}\rangle^2  \\
&= |\mathbf{P}_k\mathbf{r}|_2^2 + (1-\gamma^2)|\mathbf{Q}_k\mathbf{r}|_2^2,
\end{aligned}
\end{equation*}
where $\gamma = \frac{\langle \mathbf{Q}_k\mathbf{r},\mathbf{h}\rangle} {|\mathbf{Q}_k\mathbf{r}|_2}$.

This identity shows that SRP focuses the new update on the orthogonal innovation component $\mathbf{Q}_k\mathbf{r}$. It does not imply that the full residual energy always decreases faster than standard gradient boosting.

\begin{algorithm}[tb]
\caption{SCBoost Algorithm}
\label{alg:scboost}
\begin{algorithmic}[1]
\STATE {\bfseries Input:} Dataset $S={(x_i,y_i)}*{i=1}^{n}$, iterations $T$, learning rate $\eta$, projection threshold $\alpha$, covariance coefficient $\lambda*{\mathrm{cov}}$
\STATE \textbf{Data Split:} Partition $S$ into training set $\mathcal{D}*{train}=(\mathbf{X}*{tr},\mathbf{y}*{tr})$ and internal validation set $\mathcal{D}*{val}=(\mathbf{X}*{val},\mathbf{y}*{val})$.
\STATE Initialize logits $\mathbf{F}^{(0)}=\mathbf{0}$ on $\mathcal{D}*{train}$.
\STATE Initialize prediction history $\mathbf{H}^{(0)}\leftarrow []$.
\FOR{$t=1$ {\bfseries to} $T$}
\STATE Compute residual on $\mathcal{D}*{train}$: $\mathbf{r}^{(t)} = \mathbf{y}*{tr} - \sigma(\mathbf{F}^{(t-1)})$.
\STATE // \textit{Spectral Residual Projection (SRP)}
\IF{$t>1$}
\STATE Perform SVD: $\mathbf{H}^{(t-1)} = \mathbf{U}\mathbf{\Sigma}\mathbf{V}^{\top}$.
\STATE Select top-$k$ components $\mathbf{U}*k$ by Eq.~\eqref{eq:energy_threshold}.
\STATE Project residual: $\tilde{\mathbf{r}}^{(t)} = \mathbf{r}^{(t)} - \mathbf{U}*k(\mathbf{U}*k^\top\mathbf{r}^{(t)})$.
\ELSE
\STATE $\tilde{\mathbf{r}}^{(t)}=\mathbf{r}^{(t)}$.
\ENDIF
\STATE Train weak learner $h^{(t)}$ on $(\mathbf{X}*{tr},\tilde{\mathbf{r}}^{(t)})$.
\STATE Update logits: $\mathbf{F}^{(t)} = \mathbf{F}^{(t-1)} + \eta h^{(t)}(\mathbf{X}*{tr})$.
\STATE Update prediction history: $\mathbf{H}^{(t)} \leftarrow [\mathbf{H}^{(t-1)},h^{(t)}(\mathbf{X}*{tr})]$.
\ENDFOR
\STATE // \textit{Covariance-Regularized Weighting (CRW)}
\STATE Construct validation prediction matrix $\mathbf{H}*{val}$ with $\mathbf{H}*{val}[i,j]=h^{(j)}(x_i)$ for $x_i\in\mathbf{X}*{val}$.
\STATE Compute covariance matrix $\mathbf{C}=\frac{1}{m}\bar{\mathbf{H}}*{val}^{\top}\bar{\mathbf{H}}*{val}$.
\STATE Solve
$\min_{\mathbf{w}\in\Delta_T} \mathcal{L} \left( \mathbf{y}*{val}, \sigma(\mathbf{H}*{val}\mathbf{w}) \right) + \lambda_{\mathrm{cov}}\mathbf{w}^{\top}\mathbf{C}\mathbf{w}$.
\STATE {\bfseries Output:} Final predictor
$F(x)=\sum_{t=1}^{T}w_t h^{(t)}(x)$.
\end{algorithmic}
\end{algorithm}

\subsection{Covariance-Regularized Weighting}
\label{SCBoost: crw}

\textbf{Algorithm Description.}
SRP orthogonalizes the training target before learner induction, but the fitted learners may still be correlated because of limited learner capacity and finite-sample effects. CRW addresses this issue at the aggregation stage. Let $\mathbf{H}*{val}\in\mathbb{R}^{m\times T}$ be the validation prediction matrix, where $\mathbf{H}*{val}[i,j]=h^{(j)}(x_i), \quad x_i\in\mathbf{X}*{val}$. Let $\bar{\mathbf{H}}*{val}$ be the column-centered version of $\mathbf{H}*{val}$, and define $\mathbf{C} = \frac{1}{m}\bar{\mathbf{H}}*{val}^{\top}\bar{\mathbf{H}}*{val}$. The final ensemble uses weights $\mathbf{w}$ on the probability simplex $\Delta_T = \left(\mathbf{w}\in\mathbb{R}*+^T: \sum_{t=1}^T w_t=1 \right)$. We solve
\begin{equation}
\label{eq:weight}
\min_{\mathbf{w}\in\Delta_T} \quad \mathcal{L} \left( \mathbf{y}*{val}, \sigma(\mathbf{H}*{val}\mathbf{w}) \right) + \lambda_{\mathrm{cov}}\mathbf{w}^{\top}\mathbf{C}\mathbf{w},
\end{equation}
where $\mathcal{L}(\cdot,\cdot)$ is the logistic loss and $\lambda_{\mathrm{cov}}\geq0$ controls the strength of covariance regularization. The term $\mathbf{w}^{\top}\mathbf{C}\mathbf{w}$ penalizes covariance-weighted aggregate prediction variance on the validation set. This encourages the final ensemble to avoid placing large weights on mutually redundant predictors when doing so does not improve validation loss.

\textbf{Theoretical Motivation.}
We do not use a standalone Rademacher-complexity theorem to justify Eq.~\eqref{eq:weight}. While recent literature has deeply analyzed the high-dimensional risk and implicit bias of boosting under squared loss ($\ell_2$-boosting) \cite{su2026does}, we use the standard ambiguity decomposition as a motivation for covariance-aware aggregation. For squared loss and convex weights, define $F(x)=\sum_{t=1}^{T}w_t h^{(t)}(x)$. Then the following identity holds for each sample $(x,y)$:
\begin{equation*}
\begin{aligned}
(F(x)-y)^2 = &\sum_{t=1}^{T}w_t(h^{(t)}(x)-y)^2 \\
&- \sum_{t=1}^{T}w_t(h^{(t)}(x)-F(x))^2 .
\end{aligned}
\end{equation*}
The second term is the ambiguity term. For a fixed weighted average individual error, increasing ambiguity reduces the squared error of the ensemble. Since strongly correlated learners tend to have smaller ambiguity, the covariance penalty in Eq.~\eqref{eq:weight} is consistent with this decomposition. This argument is a motivation for CRW under squared loss; it is not a generalization guarantee for the logistic-loss objective.

\subsection{Algorithm Summary}
SCBoost differs from diversity-promoting methods such as NCL in where the diversity constraint is introduced. SCBoost modifies the target used to train each new learner, while CRW adjusts the final aggregation weights. The complete procedure is summarized in Algorithm~\ref{alg:scboost}.
\begin{itemize}
    \item \textbf{Target vs. Output:} SRP orthogonalizes the learning target on the training sample. This is different from directly penalizing correlations among model outputs during learner training.
    \item \textbf{Pre-emptive vs. Corrective:} SRP removes the selected historical component from the residual before fitting the next learner. CRW then handles remaining correlations after the learners have been trained.
    \item \textbf{Geometric Interpretation:} SRP decomposes the residual into a historical component and an orthogonal innovation component. The next learner is trained on the innovation component.
\end{itemize}

\section{Experiments and Results}
In this section, we report the experimental results regarding our proposed SCBoost algorithm in modeling real-world binary classification task in various domains. The primary objective is to evaluate its prediction performance, ensemble diversity, robustness under different noisy ratio settings, hyperparameter sensitivity analysis and ablation studies to quantify the contribution of each key component of our method.

\subsection{Setup}
\textbf{Dataset.} We used five different datasets from OpenML repository, which are all publicly available. These datasets include Madelon (2,600 samples, 500 features) \cite{guyon2007competitive}, Jasmine (2,984 samples, 144 features), Bioresponse (3,751 samples, 1,776 features), Creditcard (284,807 samples, 30 features) \cite{dalpozzolo2015calibrating}, OVA\_Uterus (1,545 samples, 10,936 features) \cite{stiglic2010stability}, QSAR (1,055 samples, 41 features), Scene (2,407 samples, 300 features), Letter (20,000 samples, 42 features), Pol (10,082 samples, 27 features) and the Elevators (16,599 samples, 17 features). We employed random undersampling \cite{hasanin2018effects, liu2020dealing} to address class imbalance in the Creditcard and OVA\_Uterus datasets.

\textbf{Baselines.} We selected a comprehensive set of seven baseline algorithms (RF \cite{breiman2001random}, ADA \cite{freund1997decision}, GBDT \cite{friedman2001greedy}, XGB \cite{chen2016xgboost}, LGBM \cite{ke2017lightgbm}, CAT \cite{prokhorenkova2018catboost}, and NGB \cite{duan2020ngboost}) that represent the major paradigms in EL and the current SOTA in gradient boosting. This allows for a thorough comparison, situating SCBoost’s performance within the broader landscape of ensemble methods. The specific parameter configurations for each baseline and SCBoost are provided in Appendix \ref{sec:appendix_parac}.

\begin{table}
  \caption{Out-of-the-box performance for SCBoost compared with the SOAT ensemble algorithms.}
  \label{tab:scboost_comparison}
  \begin{center}
    \begin{small}
    \setlength{\tabcolsep}{3pt}
    \renewcommand{\arraystretch}{0.85}
      \begin{sc}
        \begin{tabular}{llllllllll}
          \toprule
          \textbf{Dataset} & \textbf{Metric} & \textbf{RF} & \textbf{ADA} & \textbf{GBDT} & \textbf{XGB} & \textbf{LGBM} & \textbf{CAT} & \textbf{NGB} & \textbf{SCBoost} \\
          \midrule
          \multirow{3}{*}{Madelon} 
          & ACC & 0.7150$^{**}$ & 0.6112$^{**}$ & 0.7327$^{**}$ & 0.8212$^{**}$ & 0.8139$^{**}$ & \underline{0.8465}$^*$ & 0.7285$^{**}$ & \textbf{0.9396} \\
          & F1  & 0.7210$^{**}$ & 0.6118$^{**}$ & 0.7410$^{**}$ & 0.8233$^{**}$ & 0.8330$^{**}$ & \underline{0.8462}$^*$ & 0.7349$^{**}$ & \textbf{0.9385} \\
          & AUC & 0.7892$^{**}$ & 0.6525$^{**}$ & 0.8206$^{**}$ & 0.8917$^*$ & 0.9002$^*$ & \underline{0.9148}$^*$ & 0.8084$^{**}$ & \textbf{0.9649} \\
          \midrule
          \multirow{3}{*}{Jasmine} 
          & ACC & \underline{0.8150}$^{**}$ & 0.7929$^{**}$ & 0.8083$^{**}$ & 0.8026$^{**}$ & 0.8046$^{**}$ & 0.8053$^{**}$ & 0.8006$^{**}$ & \textbf{0.9381} \\
          & F1  & \underline{0.8363}$^*$ & 0.8115$^{**}$ & 0.8274$^{**}$ & 0.8179$^{**}$ & 0.8207$^{**}$ & 0.8248$^*$ & 0.8287$^{**}$ & \textbf{0.9391} \\
          & AUC & \underline{0.8809}$^*$ & 0.8434$^{**}$ & 0.8632$^*$ & 0.8608$^*$ & 0.8640$^*$ & 0.8741$^*$ & 0.8538$^{**}$ & \textbf{0.9590} \\
          \midrule
          \multirow{3}{*}{Bioresponse} 
          & ACC & \underline{0.8049}$^{**}$ & 0.7563$^{**}$ & 0.7870$^{**}$ & 0.8009$^{**}$ & 0.7985$^{**}$ & 0.7897$^{**}$ & 0.7649$^{**}$ & \textbf{0.9323} \\
          & F1  & \underline{0.8212}$^{**}$ & 0.7770$^{**}$ & 0.8078$^{**}$ & 0.8179$^{**}$ & 0.8174$^{**}$ & 0.8093$^{**}$ & 0.7903$^{**}$ & \textbf{0.9377} \\
          & AUC & \underline{0.8774}$^*$ & 0.8240$^{**}$ & 0.8543$^*$ & 0.8720$^*$ & 0.8735$^*$ & 0.8633$^*$ & 0.8320$^{**}$ & \textbf{0.9576} \\
          \midrule
          \multirow{3}{*}{Creditcard} 
          & ACC & 0.9685 & 0.9645$^*$ & 0.9614$^*$ & \underline{0.9736} & 0.9716$^*$ & 0.9675$^*$ & 0.9400$^*$ & \textbf{0.9797} \\
          & F1  & 0.9677 & 0.9643$^*$ & 0.9608$^*$ & \underline{0.9731} & 0.9711$^*$ & 0.9667$^*$ & 0.9377$^*$ & \textbf{0.9797} \\
          & AUC & 0.9923 & 0.9874$^*$ & 0.9911$^*$ & \underline{0.9933} & 0.9916 & \textbf{0.9934} & 0.9857$^*$ & 0.9896 \\
          \midrule
          \multirow{3}{*}{OVA\_Uterus} 
          & ACC & 0.9075 & 0.9078 & 0.9037$^*$ & 0.8913$^*$ & \underline{0.9358} & 0.9117 & 0.8913\#$^*$ & \textbf{0.9517} \\
          & F1  & 0.9035 & 0.9054 & 0.9039$^*$ & 0.8918$^*$ & \underline{0.9349} & 0.9110$^*$ & 0.8939\#$^*$ & \textbf{0.9544} \\
          & AUC & 0.9632 & 0.9647 & 0.9629 & 0.9644 & \textbf{0.9749} & \underline{0.9703} & 0.9578\# & \textbf{0.9749} \\
          \midrule
          \multirow{3}{*}{QSAR} 
          & ACC & \textbf{0.9341} & 0.8835$^{**}$ & \underline{0.9327} & 0.9256 & 0.9313 & 0.9243 & 0.8878$^{**}$ & 0.9326 \\
          & F1  & \textbf{0.9340} & 0.8856$^{*}$ & \underline{0.9336} & 0.9253 & 0.9314 & 0.9255 & 0.8893$^{**}$ & 0.9314 \\
          & AUC & \textbf{0.9807} & 0.9472$^*$ & 0.9702 & 0.9727 & 0.9752 & 0.9721$^*$ & 0.9482$^*$ & \underline{0.9796} \\
          \midrule
          \multirow{3}{*}{Scene} 
          & ACC & 0.9443$^{**}$ & 0.9606$^*$ & \underline{0.9803} & 0.9710 & 0.9768 & 0.9640$^*$ & 0.9640$^*$ & \textbf{0.9803} \\
          & F1  & 0.9441$^{**}$ & 0.9612$^*$ & \underline{0.9802} & 0.9712 & 0.9768 & 0.9644$^*$ & 0.9647$^*$ & \textbf{0.9803} \\
          & AUC & 0.9869 & 0.9949 & 0.9955 & 0.9946 & \textbf{0.9965} & \underline{0.9956} & 0.9898 & 0.9938 \\
          \midrule
          \multirow{3}{*}{Letter} 
          & ACC & 0.9963$^{**}$ & 0.9843$^{**}$ & 0.9928$^{**}$ & 0.9982 & \underline{0.9983} & 0.9970$^*$ & 0.9814$^{**}$ & \textbf{0.9989} \\
          & F1  & 0.9522$^{**}$ & 0.7837$^{**}$ & 0.9040$^{**}$ & 0.9775 & \underline{0.9788} & 0.9614$^*$ & 0.7032$^{**}$ & \textbf{0.9864} \\
          & AUC & 0.9998 & 0.9912$^*$ & 0.9985 & \underline{0.9999} & \textbf{0.9999} & 0.9997 & 0.9720$^{**}$ & 0.9949 \\
          \midrule
          \multirow{3}{*}{Pol} 
          & ACC & 0.9847$^{**}$ & 0.9388$^{**}$ & 0.9687$^{**}$ & 0.9888$^{**}$ & \underline{0.9896}$^*$ & 0.9876$^{**}$ & 0.9504$^{**}$ & \textbf{0.9965} \\
          & F1  & 0.9886$^{**}$ & 0.9539$^{**}$ & 0.9766$^{**}$ & 0.9916$^{**}$ & \underline{0.9922}$^*$ & 0.9907$^{**}$ & 0.9633$^{**}$ & \textbf{0.9974} \\
          & AUC & 0.9990 & 0.9849$^{**}$ & 0.9963$^*$ & \textbf{0.9994} & 0.9994 & \underline{0.9994} & 0.9859$^{**}$ & 0.9975 \\
          \midrule
          \multirow{3}{*}{Elevators} 
          & ACC & 0.8625$^{**}$ & 0.8352$^{**}$ & 0.8596$^{**}$ & \underline{0.8921}$^{**}$ & 0.8907$^{**}$ & 0.8882$^{**}$ & 0.8131$^{**}$ & \textbf{0.9678} \\
          & F1  & 0.9055$^{**}$ & 0.8884$^{**}$ & 0.9046$^{**}$ & \underline{0.9246}$^{**}$ & 0.9244$^{**}$ & 0.9230$^{**}$ & 0.8764$^{**}$ & \textbf{0.9765} \\
          & AUC & 0.9132$^*$ & 0.8852$^{**}$ & 0.9083$^*$ & \underline{0.9418}$^*$ & 0.9391$^*$ & 0.9378$^*$ & 0.8503$^{**}$ & \textbf{0.9735} \\
          \bottomrule
        \end{tabular}
      \end{sc}
    \end{small}
  \end{center}
  \vspace{2mm}
  \footnotesize
  \textbf{Note:} For the OVA\_Uterus dataset, the NGB results (\#) were obtained with \texttt{natural\_gradient=False} due to numerical instability in the official setting. All baselines marked with $*$ are significantly different from SCBoost according to the Wilcoxon signed-rank test, where $^*$ denotes $p<0.05$ and $^{**}$ denotes $p<0.01$.
\end{table}

\textbf{Metrics and Protocol.} Our experimental evaluation was designed to be comprehensive, assessing both the prediction performance and the underlying ensemble diversity of SCBoost against SOTA baselines. To rigorously measure performance, we employed a suite of standard classification metrics, including overall accuracy (ACC), the F1 score (F1) and Area Under the ROC Curve (AUC). To evaluate ensemble diversity, we employed three well-established diversity measures: Q-statistic (Q-s), Disagreement (Dis), and Ambiguity (Amb). To avoid data leakage and ensure reliable evaluation, we used Stratified ten-fold cross-validation across all experiments with strict isolation of preprocessing. Crucially, for SCBoost, we employed a nested validation strategy to optimize the CRW weights: within each training fold, we performed an internal 80/20 split. The 80\% subset was used for SRP projection and base learner training, while the held-out 20\% subset was reserved strictly for covariance estimation and weight optimization. Finally, we reported the averaged results over ten-fold cross-validation. The best results are highlighted in \textbf{bold}, and the second-best results are \underline{underlined}.

\textbf{Environmental Setting.} All experiments were conducted locally using Jupyter Notebook in an Anaconda environment on Windows 10 (Version 10.0.19045). The hardware infrastructure utilized an Intel processor (Intel64 Family 6 Model 151 Stepping 2, GenuineIntel) equipped with 20 logical cores. The software stack was built on Python 3.12.4, utilizing the following key libraries: pandas 2.2.2, numpy 1.26.4, scikit-learn 1.7.2, statsmodels 0.14.2, imbalanced-learn 0.14.0, lightgbm 4.3.0, xgboost 2.0.3, catboost 1.2.3, and NGBoost 0.5.7.

\subsection{Out-of-the-box Performance Comparison}
\label{ssec:overall_performance}

Table \ref{tab:scboost_comparison} reports the out-of-the-box performance of SCBoost and seven ensemble baselines on ten benchmark datasets. These results showed that SCBoost achieved the best or tied-best ACC and F1 on nine datasets, with clear gains on Madelon, Jasmine, Bioresponse, Pol, and Elevators. The AUC results were more mixed: SCBoost performed best or tied for best on Madelon, Jasmine, Bioresponse, OVA\_Uterus, and Elevators, but was not uniformly superior on Creditcard, QSAR, Scene, Letter, and Pol. Wilcoxon signed-rank tests indicated that many improvements were significant, while the differences were smaller or insignificant on datasets where several baselines already performed competitively. These results demonstrated the effectiveness of SCBoost in an out-of-the-box setting, where extensive hyperparameter tuning may be impractical due to limited computational resources.

\begin{table}
  \caption{Performance of SCBoost (default) against tuned SOTA ensembles.}
  \label{tab:performance_tuned}
  \begin{center}
    \begin{small}
    \setlength{\tabcolsep}{3pt}
      \begin{sc}
        \begin{tabular}{llllllllll}
          \toprule
          \textbf{Dataset} & \textbf{Metric} & \textbf{RF} & \textbf{ADA} & \textbf{GBDT} & \textbf{XGB} & \textbf{LGBM} & \textbf{CAT} & \textbf{NGB} & \textbf{SCBoost} \\
          \midrule
          \multirow{3}{*}{Madelon} 
          & ACC & 0.7104$^{**}$ & 0.6292$^{**}$ & 0.8500$^*$ & 0.8469$^*$ & 0.8415$^{**}$ & \underline{0.8681} & 0.7500\#$^{**}$ & \textbf{0.9396} \\
          & F1  & 0.7276$^{**}$ & 0.6265$^{**}$ & 0.8555 & 0.8492$^*$ & 0.8437$^*$ & \underline{0.8672} & 0.7565\#$^{**}$ & \textbf{0.9385} \\
          & AUC & 0.7855$^{**}$ & 0.6850$^{**}$ & 0.9163 & 0.9108 & 0.9088 & \underline{0.9315} & 0.8236\#$^{**}$ & \textbf{0.9649} \\
          \midrule
          \multirow{3}{*}{Jasmine} 
          & ACC & 0.8093$^{**}$ & 0.7999$^{**}$ & 0.8140$^{**}$ & 0.8170$^{**}$ & 0.8150$^*$ & \underline{0.8194}$^*$ & 0.8086\#$^{**}$ & \textbf{0.9381} \\
          & F1  & 0.8365$^*$ & 0.8200$^*$ & 0.8303$^*$ & 0.8381$^*$ & 0.8372$^*$ & \underline{0.8396}$^*$ & 0.8282\#$^*$ & \textbf{0.9391} \\
          & AUC & 0.8722 & 0.8496$^*$ & 0.8691 & 0.8637 & 0.8639$^*$ & \underline{0.8769} & 0.8597\# & \textbf{0.9590} \\
          \midrule
          \multirow{3}{*}{Bioresponse} 
          & ACC & 0.7787$^{**}$ & 0.7638$^{**}$ & 0.8030$^{**}$ & \underline{0.8054}$^{**}$ & 0.8040$^{**}$ & 0.8022$^{**}$ & 0.7926\#$^{**}$ & \textbf{0.9323} \\
          & F1  & 0.7972$^{**}$ & 0.7838$^{**}$ & 0.8210$^{**}$ & \underline{0.8230}$^{**}$ & 0.8213$^{**}$ & 0.8215$^{**}$ & 0.8129\#$^{**}$ & \textbf{0.9377} \\
          & AUC & 0.8516$^*$ & 0.8271$^{**}$ & 0.8727 & \underline{0.8765} & 0.8731 & 0.8710 & 0.8605\# & \textbf{0.9576} \\
          \midrule
          \multirow{3}{*}{Creditcard} 
          & ACC & 0.9593 & 0.9685 & 0.9705 & 0.9736 & \underline{0.9787} & 0.9726 & 0.9583\# & \textbf{0.9797} \\
          & F1  & 0.9579 & 0.9682 & 0.9699 & 0.9731 & \underline{0.9783} & 0.9721 & 0.9582\# & \textbf{0.9797} \\
          & AUC & 0.9905 & 0.9876 & 0.9910 & \underline{0.9935} & \textbf{0.9943} & 0.9915 & 0.9909\# & 0.9896 \\
          \midrule
          \multirow{3}{*}{OVA\_Uterus} 
          & ACC & 0.9318 & 0.9278 & 0.9318 & 0.9195 & 0.9318 & \underline{0.9358} & 0.9195\# & \textbf{0.9517} \\
          & F1  & 0.9325 & 0.9278 & 0.9316 & 0.9196 & 0.9309 & \underline{0.9362} & 0.9199\# & \textbf{0.9544} \\
          & AUC & 0.9672 & \underline{0.9731} & 0.9703 & 0.9649 & 0.9718 & 0.9691$^*$ & 0.9587\#$^{**}$ & \textbf{0.9749} \\
          \bottomrule
        \end{tabular}
      \end{sc}
    \end{small}
  \end{center}
  \vskip -0.1in
  \vspace{2mm}
  \footnotesize
  \textbf{Note:} The NGB results (\#) were obtained with \texttt{natural\_gradient=False} due to numerical instability in the official setting. See Appendix~\ref{sec:appendix_parac} for details. All baselines marked with $*$ are significantly different from SCBoost according to the Wilcoxon signed-rank test, where $^*$ denotes $p<0.05$ and $^{**}$ denotes $p<0.01$.
\end{table}

\begin{table}
  \caption{Diversity metrics of ensemble methods across benchmark datasets.}
  \label{tab:diversity}
  \begin{center}
    \begin{small}
    \setlength{\tabcolsep}{1.2pt}
      \begin{sc}
        \begin{tabular}{llllllllll}
          \toprule
          \textbf{Dataset} & \textbf{Metric} & \textbf{RF} & \textbf{ADA} & \textbf{GBDT} & \textbf{XGB} & \textbf{LGBM} & \textbf{CAT} & \textbf{NGB} & \textbf{SCBoost} \\
          \midrule
          \multirow{3}{*}{\textbf{Madelon}} 
          & Dis ($\uparrow$) & 0.4734$^{***}$ & \underline{0.4962}$^{***}$ & 0.0869\# & 0.2047$^{***}$ & 0.2047$^{***}$ & 0.397$^{***}$ & 0.0483$^{***}$ & \textbf{0.9900} \\
          & Q-s ($\downarrow$) & 0.0858 & \textbf{-0.0006} & \underline{-0.3039}\# & 0.8669 & 0.8669 & 0.9851 & 0.9829 & 0.0000 \\
          & Amb ($\uparrow$) & \underline{0.1572}$^{***}$ & 0.0625$^{***}$ & 0.0086 & 0.0368$^{***}$ & 0.0368$^{***}$ & 0.0030$^{***}$ & 0.0023$^{***}$ & \textbf{0.5004} \\
          \midrule
          \multirow{3}{*}{\textbf{Jasmine}} 
          & Dis ($\uparrow$) & 0.2871$^{***}$ & \underline{0.4593}$^{***}$ & 0.1105 & 0.0913$^{***}$ & 0.0913$^{***}$ & 0.0335$^{***}$ & 0.0212$^{***}$ & \textbf{0.9900} \\
          & Q-s ($\downarrow$) & 0.7054 & 0.1761 & \textbf{-0.0886}\# & 0.9687 & 0.9687 & 0.9875 & 0.9890 & \underline{0.0000} \\
          & Amb ($\uparrow$) & 0.0898$^{***}$ & \underline{0.1970}$^{***}$ & 0.0400 & 0.0094$^{***}$ & 0.0094$^{***}$ & 0.0027$^{***}$ & 0.0019$^{***}$ & \textbf{0.5000} \\
          \midrule
          \multirow{3}{*}{\textbf{Bioresponse}} 
          & Dis ($\uparrow$) & 0.3418$^{***}$ & \underline{0.4974}$^{***}$ & 0.0705\# & 0.0757$^{***}$ & 0.0757$^{***}$ & 0.0467$^{***}$ & 0.0617$^{***}$ & \textbf{0.9900} \\
          & Q-s ($\downarrow$) & 0.5505 & 0.0553 & \textbf{-0.2129}\# & 0.9768 & 0.9768 & 0.9832 & 0.8929 & \underline{0.0000} \\
          & Amb ($\uparrow$) & 0.1205$^{***}$ & \underline{0.1720}$^{***}$ & 0.0080 & 0.0089$^{***}$ & 0.0089$^{***}$ & 0.0005$^{***}$ & 0.0142$^{***}$ & \textbf{0.4577} \\
          \bottomrule
        \end{tabular}
      \end{sc}
    \end{small}
  \end{center}
  \vspace{2mm}
  \footnotesize
  \textbf{Note:} If SCBoost and some baselines yield identical fold-level values, resulting in zero paired differences. In such cases, the Wilcoxon signed-rank test is not applicable (\#). All baselines marked with $*$ are significantly different from SCBoost according to the Wilcoxon signed-rank test, where $^*$ denotes $p<0.05$, $^{**}$ denotes $p<0.01$, and $^{***}$ denotes $p<0.001$. 
\end{table}

\subsection{Comparison with Tuned SOTA Baselines}
\label{ssec:appendix_tuned_performance}

Table \ref{tab:performance_tuned} presents the comparison between SCBoost with fixed default parameters and seven tuned ensemble baselines on five benchmark datasets. These results reported that SCBoost achieved the best ACC and F1 on all five datasets, and obtained the best AUC on four of them. The gains were clear on Madelon, Jasmine, and Bioresponse, where SCBoost improved over the strongest tuned baseline by a large margin in ACC and F1. On Creditcard, SCBoost achieved the best ACC and F1, while its AUC was lower than that of LGBM and XGB. On OVA\_Uterus, SCBoost obtained the best ACC, F1, and AUC, although the differences in AUC were small. Although tuning improved most baselines, SCBoost still achieved the strongest overall ACC and F1 without dataset-specific tuning, suggesting that its gains were not merely replaceable by hyperparameter optimization.

\subsection{Ensemble Diversity Analysis}
\label{ssec:diversity_analysis}
Table \ref{tab:diversity} summarizes the diversity metrics on three benchmark datasets. These results exhibited taht SCBoost obtained higher Dis and Amb scores than the baselines, while its Q-s values were close to zero. These results were consistent with the empirical effect of SRP: the residual target was projected onto the orthogonal complement of the selected historical prediction subspace before training the next learner, which reduced components already represented by previous learners. This does not guarantee exact orthogonality of the fitted learners, since the base learner may only approximate the projected target. The higher Amb scores were also consistent with the role of CRW, which penalized validation-set covariance during weight optimization. Beyond statistical metrics, SCBoost exhibited a significantly flatter singular value spectrum in its prediction history compared to GBDT in Appendix \ref{sec:redun_removal}. This slower spectral decay provided direct geometric evidence that SRP effectively forces learners to span a broader, less redundant prediction subspace. Overall, this study indicates that SRP and CRW reduced empirical learner redundancy under the reported diversity metrics. 

\begin{figure}
  \begin{center}
    \centerline{\includegraphics[width=\columnwidth]{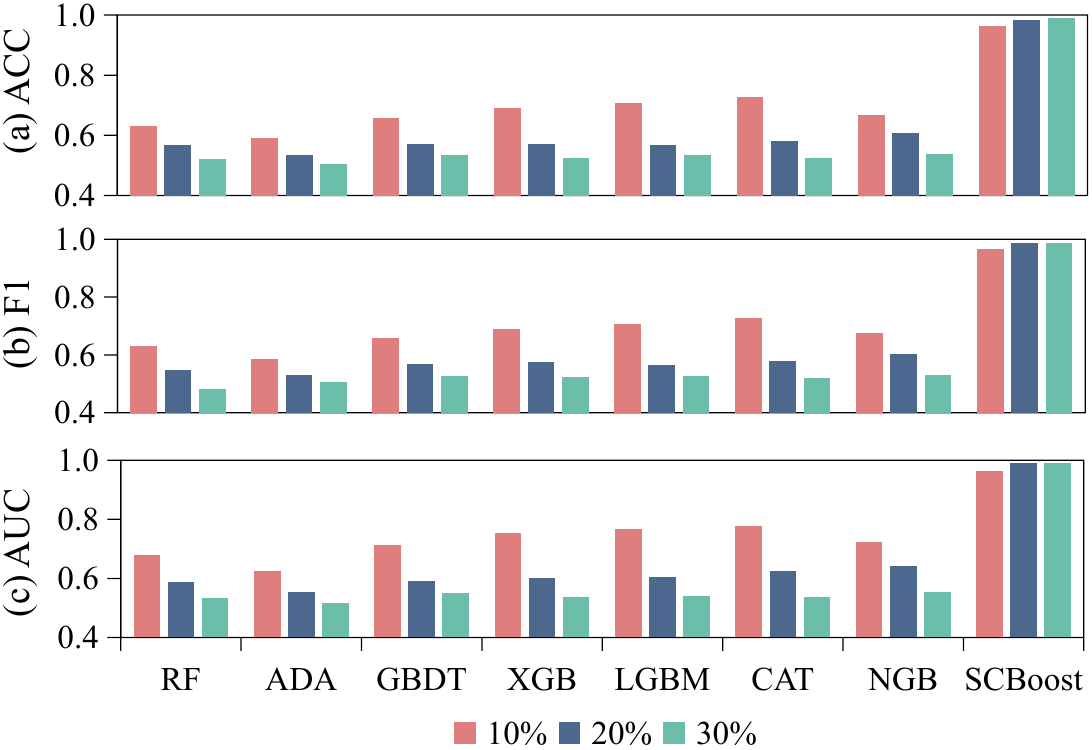}}
    \caption{Robustness evaluation for SCBoost and baselines against increasing levels (10\%-30\%) of label noise based on Madelon dataset.}
    \label{fig:2}
  \end{center}
\end{figure}

\begin{figure}
  \begin{center}
    \centerline{\includegraphics[width=\columnwidth]{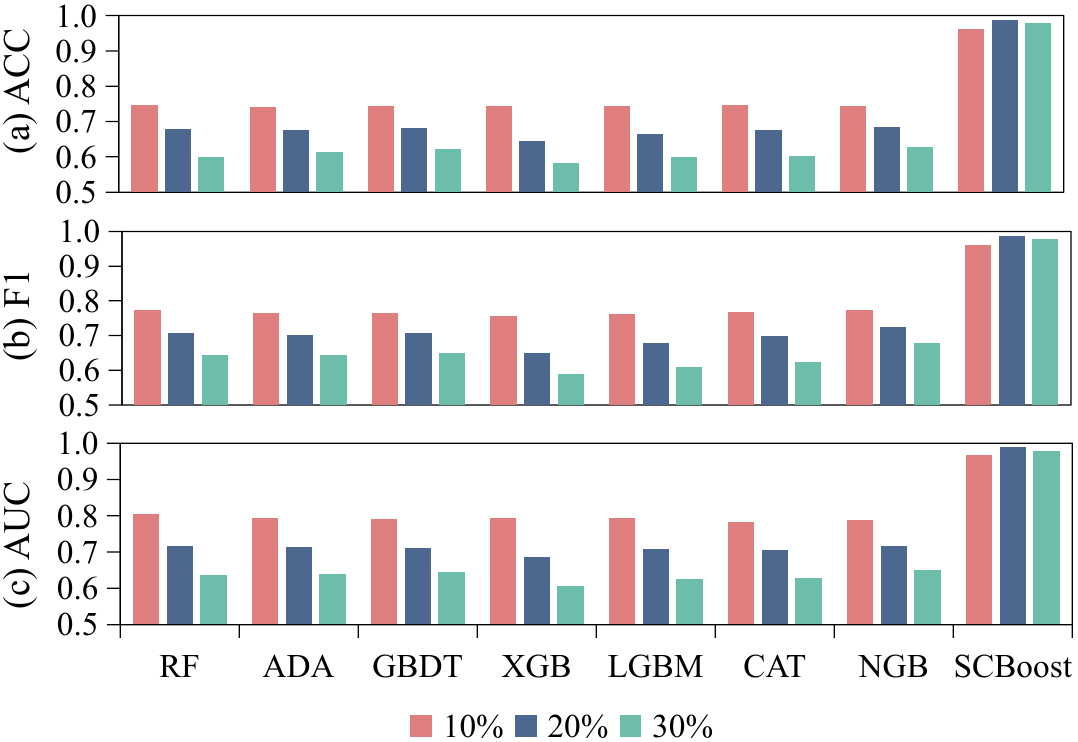}}
    \caption{Robustness evaluation for SCBoost and baselines against increasing levels (10\%-30\%) of label noise based on Jasmine dataset.}
    \label{fig:3}
  \end{center}
\end{figure}

\subsection{Robustness Analysis}
\label{ssec:robustness_analysis}
Figure \ref{fig:2} and Figure \ref{fig:3} reports the robustness results under 10\%-30\% label-flip noise on Madelon and Jasmine. These results showed that SCBoost maintained competitive performance as the noise ratio increased, while several baselines showed larger degradation. These results were consistent with the motivation of SRP: instead of fitting the full residual $\mathbf{r}^{(t)}$, SCBoost trained each new learner on the projected target $\tilde{\mathbf{r}}^{(t)}$, which removed the component lying in the selected historical prediction subspace. The fixed-subspace analysis in Proposition~\ref{prop:noise_snr} showed that such a projection reduced the expected energy of isotropic noise under explicit independence assumptions, and improved SNR only when the removed signal fraction was smaller than the removed noise fraction. Since the robustness experiments used label-flip noise rather than additive Gaussian noise, these results should be interpreted as empirical evidence of noise robustness, not as a direct validation of the fixed-subspace SNR analysis.

\begin{figure*}
    \centering
    \includegraphics[width=\textwidth]{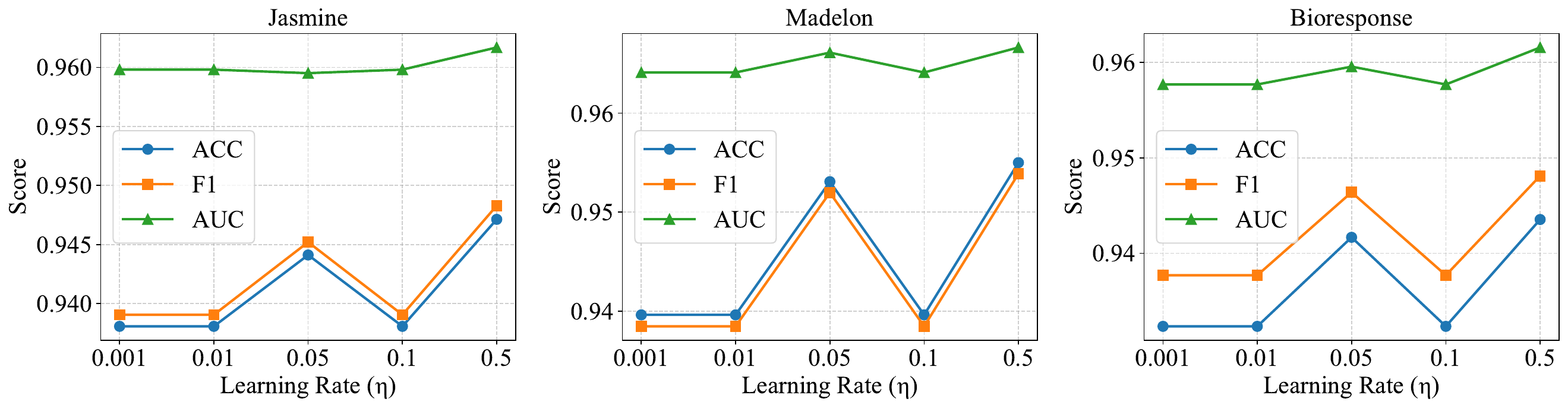}
    \caption{Learning rate sensitivity analysis of SCBoost on Jasmine, Madelon, and Bioresponse. We evaluated $\eta \in {0.001, 0.01, 0.05, 0.1, 0.5}$ and reported ACC, F1, and AUC for each dataset.}
    \label{fig:lr_sensitivity}
\end{figure*}

\begin{table}
  \caption{Ablation of SCBoost algorithm based on three datasets.}
  \label{tab:ablation_resized}
  \begin{center}
    \begin{small}
    \setlength{\tabcolsep}{3pt}
      \begin{sc}
        \begin{tabular}{ccccccc}
          \toprule
          \textbf{Dataset} & \textbf{Metric} & \textbf{SRP} & \textbf{CRW} & \textbf{SRP+CRW} \\
          \midrule
          \multirow{3}{*}{\textbf{Madelon}} 
          & ACC & 0.9285 & 0.9158 & \textbf{0.9396} \\
          & F1 & 0.9257 & 0.9152 & \textbf{0.9385} \\
          & AUC & 0.9467 & 0.9640 & \textbf{0.9649} \\
          \midrule
          \multirow{3}{*}{\textbf{Jasmine}} 
          & ACC & 0.9143 & 0.8643 & \textbf{0.9381} \\
          & F1 & 0.9135 & 0.8759 & \textbf{0.9391} \\
          & AUC & 0.9367 & 0.9277 & \textbf{0.9590} \\
          \midrule
          \multirow{3}{*}{\textbf{Bioresponse}} 
          & ACC & 0.8628 & 0.8537 & \textbf{0.9323} \\
          & F1 & 0.8690 & 0.8686 & \textbf{0.9377} \\
          & AUC & 0.9025 & 0.9153 & \textbf{0.9576} \\
          \bottomrule
        \end{tabular}
      \end{sc}
    \end{small}
  \end{center}
\end{table}

\subsection{Hyperparameter Sensitivity Analysis}
\label{ssec:lr_sensitivity}
Figure \ref{fig:lr_sensitivity} reports the sensitivity of SCBoost to the learning rate $\eta$ on Jasmine, Madelon, and Bioresponse. We evaluated $\eta \in {0.001, 0.01, 0.05, 0.1, 0.5}$ using ACC, F1, and AUC. Overall, SCBoost remained stable across different learning rates, especially in terms of AUC. On all three datasets, larger learning rates generally improved ACC and F1, with $\eta=0.5$ giving the best or near-best results. In contrast, $\eta=0.1$ led to a small performance drop on Madelon and Bioresponse, suggesting that the effect of $\eta$ was not strictly monotonic. These results indicated that SCBoost was not highly sensitive to small learning rates and that a moderately large learning rate could improve empirical performance under the default setting. Additionally, Appendix \ref{sec:appendix_ablation} showed that shallow trees underfit orthogonal innovations, while deep trees ($D>5$) overfit without improving the redundancy ratio. Therefore, SCBoost optimally requires moderate-depth learners.

\subsection{Ablation}
\label{ssec:ablation}
Table \ref{tab:ablation_resized} presents the ablation results on Madelon, Jasmine, and Bioresponse. SRP alone outperformed CRW alone, suggesting that modifying the residual target was more influential than only reweighting the trained learners. Combining SRP and CRW achieved the best performance on all three datasets. For example, on Bioresponse, SRP+CRW improved ACC by more than 7\% compared with either component alone. These results were consistent with the roles of the two components: SRP removed the selected historical component from the residual target as described in Proposition~\ref{prop:ortho}, while CRW reduced the influence of highly correlated learners through validation-set covariance regularization. Overall, the ablation results indicated that both target projection and covariance-aware aggregation contributed to the final performance.

\section{Conclusion}
In this work, we re-examined learner redundancy in gradient boosting and proposed a shift from traditional residual fitting to explicit residual orthogonalization. Our proposed framework, SCBoost, achieves this through a dual mechanism: purifying the learning target via Spectral Residual Projection (SRP) and penalizing ensemble correlation via Covariance-Regularized Weighting (CRW). We theoretically characterized SRP's ability to isolate empirical innovations and improve the effective noise level, which was corroborated by SCBoost's strong default performance across multiple benchmark datasets. Although we observed some metric-specific variations that depend on data characteristics, our findings clearly demonstrate the value of geometric target modification in ensemble learning. Future work will focus on designing more computationally scalable projection techniques, enabling this non-redundant boosting paradigm to be deployed on larger-scale industrial applications.

\section*{Acknowledgements}
This work was supported in part by the National Natural Science Foundation of China under Grant 62403043. The authors declare that they have no known competing financial interests or personal relationships that could have appeared to influence the work reported in this paper.

\bibliography{Sample}

\newpage

\appendix

\section{Proof of Proposition \ref{prop:ortho}}
\label{app:sec:a}

\begin{proof}
Since $\mathbf{U}_k^\top\mathbf{U}_k=\mathbf{I}_k$, we have
\begin{equation*}
\begin{aligned}
\mathbf{P}_k^2 &=(\mathbf{U}_k\mathbf{U}_k^\top)(\mathbf{U}_k\mathbf{U}_k^\top) \\
&=\mathbf{U}_k(\mathbf{U}_k^\top\mathbf{U}_k)\mathbf{U}_k^\top \\
&=\mathbf{U}_k\mathbf{I}_k\mathbf{U}_k^\top=\mathbf{P}_k,
\end{aligned}
\end{equation*}
and
\begin{equation*}
\mathbf{P}_k^\top=(\mathbf{U}_k\mathbf{U}_k^\top)^\top=\mathbf{U}_k\mathbf{U}_k^\top=\mathbf{P}_k.
\end{equation*}
Thus $\mathbf{P}_k$ is the orthogonal projector onto $\mathcal{H}_{t-1}=\operatorname{span}(\mathbf{U}_k)$. Moreover,
\begin{equation*}
\begin{aligned}
\mathbf{Q}_k^2 &=(\mathbf{I}_n-\mathbf{P}_k)^2 \\
&=\mathbf{I}_n-2\mathbf{P}_k+\mathbf{P}_k^2 \\
&=\mathbf{I}_n-\mathbf{P}_k=\mathbf{Q}_k,
\end{aligned}
\end{equation*}
and
\begin{equation*}
\mathbf{Q}_k^\top=(\mathbf{I}_n-\mathbf{P}_k)^\top=\mathbf{I}_n-\mathbf{P}_k^\top=\mathbf{I}_n-\mathbf{P}_k=\mathbf{Q}_k.
\end{equation*}
Hence $\mathbf{Q}_k$ is the orthogonal projector onto $\mathcal{H}_{t-1}^{\perp}$.

For any feasible $\mathbf{z}$, $\mathbf{z}\in\mathcal{H}_{t-1}^{\perp}$, so $\mathbf{Q}_k\mathbf{z}=\mathbf{z}$ and $\mathbf{P}_k\mathbf{z}=\mathbf{0}$. Since
\begin{equation*}
\mathbf{r}^{(t)}=\mathbf{P}_k\mathbf{r}^{(t)}+\mathbf{Q}_k\mathbf{r}^{(t)},
\end{equation*}
we obtain
\begin{equation*}
\begin{aligned}
|\mathbf{z}-\mathbf{r}^{(t)}|_2^2 &=|\mathbf{z}-\mathbf{P}_k\mathbf{r}^{(t)}-\mathbf{Q}_k\mathbf{r}^{(t)}|_2^2 \\
&=|(\mathbf{z}-\mathbf{Q}_k\mathbf{r}^{(t)})-\mathbf{P}_k\mathbf{r}^{(t)}|_2^2.
\end{aligned}
\end{equation*}
Because $\mathbf{z}-\mathbf{Q}_k\mathbf{r}^{(t)}\in\mathcal{H}_{t-1}^{\perp}$ and $\mathbf{P}_k\mathbf{r}^{(t)}\in\mathcal{H}_{t-1}$,
\begin{equation*}
\langle \mathbf{z}-\mathbf{Q}_k\mathbf{r}^{(t)},\mathbf{P}_k\mathbf{r}^{(t)}\rangle=0.
\end{equation*}
Therefore,
\begin{equation*}
|\mathbf{z}-\mathbf{r}^{(t)}|_2^2=|\mathbf{z}-\mathbf{Q}_k\mathbf{r}^{(t)}|_2^2+|\mathbf{P}_k\mathbf{r}^{(t)}|_2^2.
\end{equation*}
The second term is independent of $\mathbf{z}$, and the first term is uniquely minimized at $\mathbf{z}=\mathbf{Q}_k\mathbf{r}^{(t)}$. Hence
\begin{equation*}
\begin{aligned}
\tilde{\mathbf{r}}^{(t)} &=\mathbf{Q}_k\mathbf{r}^{(t)}=\mathop{\arg\min}_{\mathbf{z}\in\mathbb{R}^n}|\mathbf{z}-\mathbf{r}^{(t)}|_2^2 \\
&\quad \text{s.t.}\quad \langle \mathbf{z},\mathbf{u}\rangle=0,\ \forall \mathbf{u}\in\mathcal{H}_{t-1}.
\end{aligned}
\end{equation*}

Next,
\begin{equation*}
\begin{aligned}
|\mathbf{r}^{(t)}|_2^2 &=|\mathbf{P}_k\mathbf{r}^{(t)}+\mathbf{Q}_k\mathbf{r}^{(t)}|_2^2 \\
&=|\mathbf{P}_k\mathbf{r}^{(t)}|_2^2+|\mathbf{Q}_k\mathbf{r}^{(t)}|_2^2+2\langle \mathbf{P}_k\mathbf{r}^{(t)},\mathbf{Q}_k\mathbf{r}^{(t)}\rangle.
\end{aligned}
\end{equation*}
The cross term satisfies
\begin{equation*}
\begin{aligned}
\langle \mathbf{P}_k\mathbf{r}^{(t)},\mathbf{Q}_k\mathbf{r}^{(t)}\rangle &=(\mathbf{r}^{(t)})^\top\mathbf{P}_k^\top\mathbf{Q}_k\mathbf{r}^{(t)} \\
&=(\mathbf{r}^{(t)})^\top\mathbf{P}_k(\mathbf{I}_n-\mathbf{P}_k)\mathbf{r}^{(t)} \\
&=(\mathbf{r}^{(t)})^\top(\mathbf{P}_k-\mathbf{P}_k^2)\mathbf{r}^{(t)}=0.
\end{aligned}
\end{equation*}
Thus
\begin{equation*}
|\mathbf{r}^{(t)}|_2^2=|\mathbf{P}_k\mathbf{r}^{(t)}|_2^2+|\mathbf{Q}_k\mathbf{r}^{(t)}|_2^2.
\end{equation*}
Since $\mathbf{P}_k\mathbf{r}^{(t)}=\sum_{i=1}^{k}(\mathbf{u}_i^\top\mathbf{r}^{(t)})\mathbf{u}_i$ and $\{\mathbf{u}_i\}_{i=1}^k$ is orthonormal,
\begin{equation*}
\begin{aligned}
|\mathbf{P}_k\mathbf{r}^{(t)}|_2^2 &=\left|\sum_{i=1}^{k}(\mathbf{u}_i^\top\mathbf{r}^{(t)})\mathbf{u}_i\right|_2^2 \\
&=\sum_{i=1}^{k}(\mathbf{u}_i^\top\mathbf{r}^{(t)})^2.
\end{aligned}
\end{equation*}
Therefore,
\begin{equation*}
\begin{aligned}
|\tilde{\mathbf{r}}^{(t)}|_2^2 &=|\mathbf{Q}_k\mathbf{r}^{(t)}|_2^2 \\
&=|\mathbf{r}^{(t)}|_2^2-|\mathbf{P}_k\mathbf{r}^{(t)}|_2^2 \\
&=|\mathbf{r}^{(t)}|_2^2-\sum_{i=1}^{k}(\mathbf{u}_i^\top\mathbf{r}^{(t)})^2.
\end{aligned}
\end{equation*}
\end{proof}

\section{Proof of Proposition \ref{prop:noise_snr}}
\label{app:sec:b}

\begin{proof}
Since $\mathbf{P}_k$ is a fixed rank-$k$ orthogonal projector, $\mathbf{P}_k^\top=\mathbf{P}_k$, $\mathbf{P}_k^2=\mathbf{P}_k$, and $\operatorname{rank}(\mathbf{P}_k)=k$. Hence $\mathbf{Q}_k=\mathbf{I}_n-\mathbf{P}_k$ satisfies
\begin{equation*}
\begin{aligned}
\mathbf{Q}_k^\top &=\mathbf{Q}_k,\\
\mathbf{Q}_k^2 &=\mathbf{Q}_k,\\
\operatorname{rank}(\mathbf{Q}_k) &=n-k=d.
\end{aligned}
\end{equation*}
Therefore, there exists an orthogonal matrix $\mathbf{R}\in\mathbb{R}^{n\times n}$ such that
\begin{equation*}
\mathbf{Q}_k=\mathbf{R}\begin{bmatrix}\mathbf{I}_d&\mathbf{0}\\ \mathbf{0}&\mathbf{0}\end{bmatrix}\mathbf{R}^\top.
\end{equation*}
Let $\mathbf{g}=\nu^{-1}\mathbf{R}^\top\boldsymbol{\epsilon}$. Since $\boldsymbol{\epsilon}\sim\mathcal{N}(\mathbf{0},\nu^2\mathbf{I}_n)$ and $\mathbf{R}$ is orthogonal, $\mathbf{g}\sim\mathcal{N}(\mathbf{0},\mathbf{I}_n)$. Thus
\begin{equation*}
\begin{aligned}
|\mathbf{Q}_k\boldsymbol{\epsilon}|_2^2 &=\boldsymbol{\epsilon}^\top\mathbf{Q}_k^\top\mathbf{Q}_k\boldsymbol{\epsilon}=\boldsymbol{\epsilon}^\top\mathbf{Q}_k\boldsymbol{\epsilon} \\
&=\nu^2\mathbf{g}^\top\begin{bmatrix}\mathbf{I}_d&\mathbf{0}\\ \mathbf{0}&\mathbf{0}\end{bmatrix}\mathbf{g} \\
&=\nu^2\sum_{i=1}^{d}g_i^2.
\end{aligned}
\end{equation*}
Since $g_i\overset{i.i.d.}{\sim}\mathcal{N}(0,1)$, $\sum_{i=1}^{d}g_i^2\sim\chi_d^2$. Therefore,
\begin{equation*}
\begin{aligned}
\mathbb{E}|\mathbf{Q}_k\boldsymbol{\epsilon}|_2^2 &=\nu^2\mathbb{E}\left[\sum_{i=1}^{d}g_i^2\right] \\
&=\nu^2\sum_{i=1}^{d}\mathbb{E}[g_i^2]=d\nu^2.
\end{aligned}
\end{equation*}
Also,
\begin{equation*}
\mathbb{E}|\boldsymbol{\epsilon}|_2^2=\sum_{i=1}^{n}\mathbb{E}[\epsilon_i^2]=n\nu^2.
\end{equation*}
Hence
\begin{equation*}
\begin{aligned}
\mathbb{E}|\mathbf{Q}_k\boldsymbol{\epsilon}|_2^2 &=d\nu^2=(n-k)\nu^2 \\
&=\left(1-\frac{k}{n}\right)n\nu^2 \\
&=\left(1-\frac{k}{n}\right)\mathbb{E}|\boldsymbol{\epsilon}|_2^2.
\end{aligned}
\end{equation*}

Let $X=\sum_{i=1}^{d}g_i^2\sim\chi_d^2$. For any $x>0$,
\begin{equation*}
\mathbb{P}\left(X\ge d+2\sqrt{dx}+2x\right)\le e^{-x}.
\end{equation*}
Taking $x=\log(1/\delta)$ gives
\begin{equation*}
\mathbb{P}\left(X\le d+2\sqrt{d\log(1/\delta)}+2\log(1/\delta)\right)\ge 1-\delta.
\end{equation*}
Since $|\mathbf{Q}_k\boldsymbol{\epsilon}|_2^2=\nu^2X$, with probability at least $1-\delta$,
\begin{equation*}
|\mathbf{Q}_k\boldsymbol{\epsilon}|_2^2\le \nu^2\left[d+2\sqrt{d\log(1/\delta)}+2\log(1/\delta)\right].
\end{equation*}

Next, since $\mathbf{r}=\mathbf{s}+\boldsymbol{\epsilon}$,
\begin{equation*}
\mathbf{Q}_k\mathbf{r}=\mathbf{Q}_k\mathbf{s}+\mathbf{Q}_k\boldsymbol{\epsilon}.
\end{equation*}
Let $\mathbf{a}=\mathbf{Q}_k\mathbf{s}$ and $\mathbf{b}=\mathbf{Q}_k\boldsymbol{\epsilon}$. Then
\begin{equation*}
|\mathbf{Q}_k\mathbf{r}|_2^2=|\mathbf{a}+\mathbf{b}|_2^2=|\mathbf{a}|_2^2+2\langle\mathbf{a},\mathbf{b}\rangle+|\mathbf{b}|_2^2.
\end{equation*}
For any $\eta>0$,
\begin{equation*}
\begin{aligned}
0 &\le |\sqrt{\eta}\mathbf{a}-\eta^{-1/2}\mathbf{b}|_2^2 \\
&=\eta|\mathbf{a}|_2^2-2\langle\mathbf{a},\mathbf{b}\rangle+\eta^{-1}|\mathbf{b}|_2^2.
\end{aligned}
\end{equation*}
Thus
\begin{equation*}
2\langle\mathbf{a},\mathbf{b}\rangle\le \eta|\mathbf{a}|_2^2+\eta^{-1}|\mathbf{b}|_2^2.
\end{equation*}
Therefore,
\begin{equation*}
\begin{aligned}
|\mathbf{Q}_k\mathbf{r}|_2^2 &\le (1+\eta)|\mathbf{a}|_2^2+(1+\eta^{-1})|\mathbf{b}|_2^2 \\
&=(1+\eta)|\mathbf{Q}_k\mathbf{s}|_2^2+(1+\eta^{-1})|\mathbf{Q}_k\boldsymbol{\epsilon}|_2^2.
\end{aligned}
\end{equation*}

Finally, assume $0\le k<n$ and $\mathbf{s}\neq\mathbf{0}$. Since $\mathbf{P}_k$ and $\mathbf{Q}_k$ are orthogonal complementary projectors,
\begin{equation*}
|\mathbf{s}|_2^2=|\mathbf{P}_k\mathbf{s}|_2^2+|\mathbf{Q}_k\mathbf{s}|_2^2.
\end{equation*}
By $\rho=|\mathbf{P}_k\mathbf{s}|_2^2/|\mathbf{s}|_2^2$, we have
\begin{equation*}
|\mathbf{Q}_k\mathbf{s}|_2^2=|\mathbf{s}|_2^2-|\mathbf{P}_k\mathbf{s}|_2^2=(1-\rho)|\mathbf{s}|_2^2.
\end{equation*}
Therefore,
\begin{equation*}
\begin{aligned}
\frac{\operatorname{SNR}(\mathbf{Q}_k\mathbf{r})}{\operatorname{SNR}(\mathbf{r})} &=\frac{|\mathbf{Q}_k\mathbf{s}|_2^2/((n-k)\nu^2)}{|\mathbf{s}|_2^2/(n\nu^2)} \\
&=\frac{n|\mathbf{Q}_k\mathbf{s}|_2^2}{(n-k)|\mathbf{s}|_2^2} \\
&=\frac{n(1-\rho)|\mathbf{s}|_2^2}{(n-k)|\mathbf{s}|_2^2}=\frac{1-\rho}{1-k/n}.
\end{aligned}
\end{equation*}
Thus
\begin{equation*}
\begin{aligned}
&\operatorname{SNR}(\mathbf{Q}_k\mathbf{r})>\operatorname{SNR}(\mathbf{r}) \\
\Longleftrightarrow\quad &\frac{1-\rho}{1-k/n}>1 \\
\Longleftrightarrow\quad &1-\rho>1-\frac{k}{n} \quad\Longleftrightarrow\quad \rho<\frac{k}{n}.
\end{aligned}
\end{equation*}
\end{proof}

\section{Spectral Evidence of Prediction Diversity}
\label{sec:redun_removal}

To examine the structure of the learned ensemble, Figure~\ref{fig:Spectral_Decay} compared the normalized singular value spectra of the prediction history matrices for SCBoost and GBDT on the Jasmine dataset. GBDT showed a faster spectral decay, indicating that its prediction history was more concentrated in the leading singular directions. In contrast, SCBoost produced a flatter average spectrum across 10-fold cross-validation, suggesting that its learners covered a broader empirical prediction subspace. This observation was consistent with the role of SRP, which removes the selected historical component from the residual target before fitting the next learner. The result should be interpreted as empirical evidence of reduced prediction redundancy, not as a direct validation of the fixed-subspace SNR analysis.

\begin{figure}[ht]
\begin{center}
\centerline{\includegraphics[width=\columnwidth]{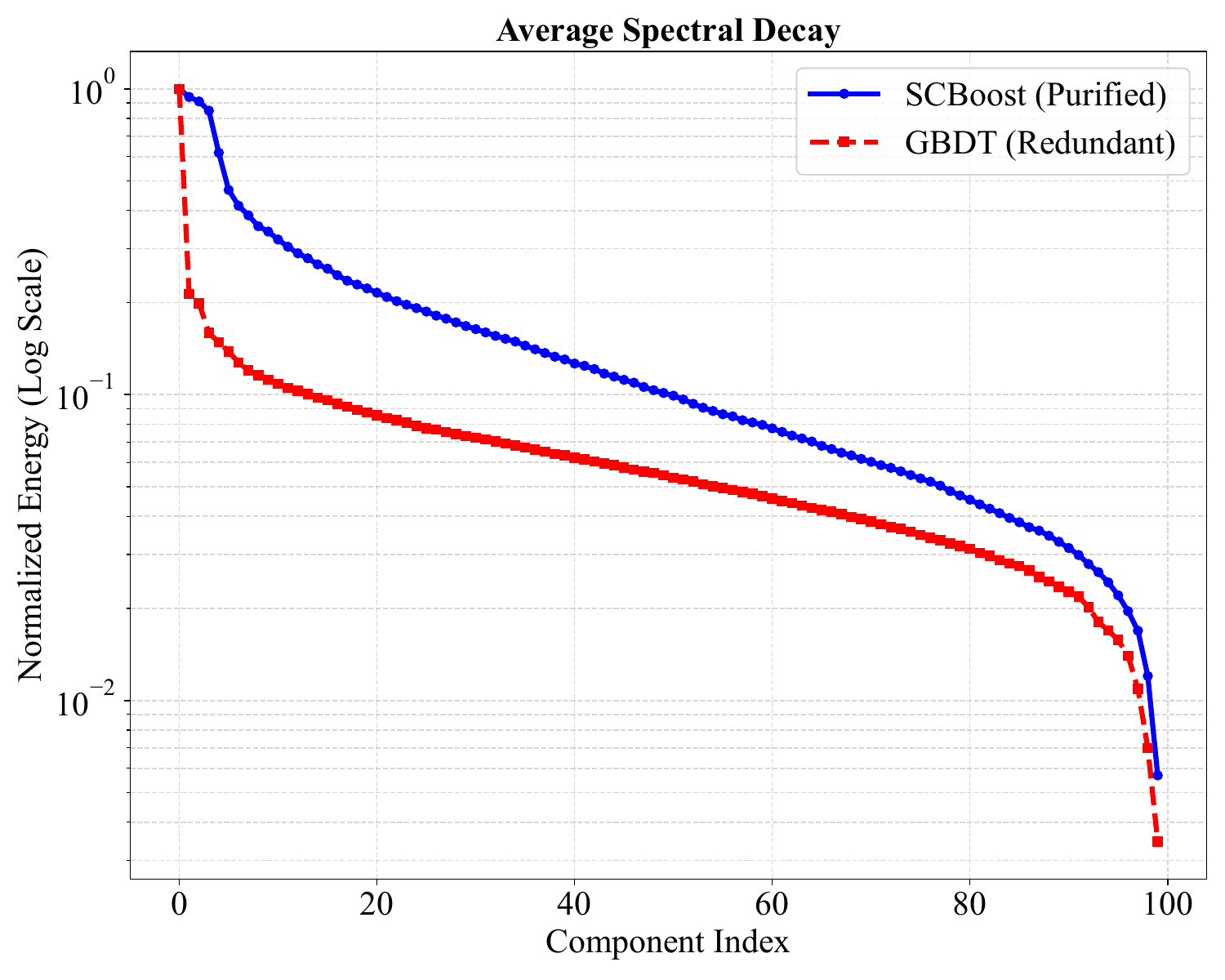}}
\caption{Normalized singular value spectra of prediction history matrices for SCBoost and GBDT on the Jasmine dataset, averaged over 10-fold cross-validation.}
\label{fig:Spectral_Decay}
\end{center}
\vskip -0.2in
\end{figure}

\section{Effect of Base Learner Capacity}
\label{sec:appendix_ablation}

We further evaluated the effect of base learner depth on SCBoost using the Jasmine dataset. The energy threshold was fixed to $\alpha=0.9$, and the tree depth was varied over $D\in[1,8]$. Figure~\ref{fig:Capacity_vs_Approximation_Gap} reported test accuracy and the redundancy ratio. Accuracy increased from shallow trees to moderate depths and reached its best value at $D=5$. The redundancy ratio decreased over the same range, suggesting that moderate-capacity learners better fitted the projected residual targets. When the depth was further increased, accuracy did not continue to improve. These results indicated that SCBoost benefited from sufficient learner capacity, but deeper trees did not necessarily provide better generalization.

\begin{figure}[ht]
\begin{center}
\centerline{\includegraphics[width=\columnwidth]{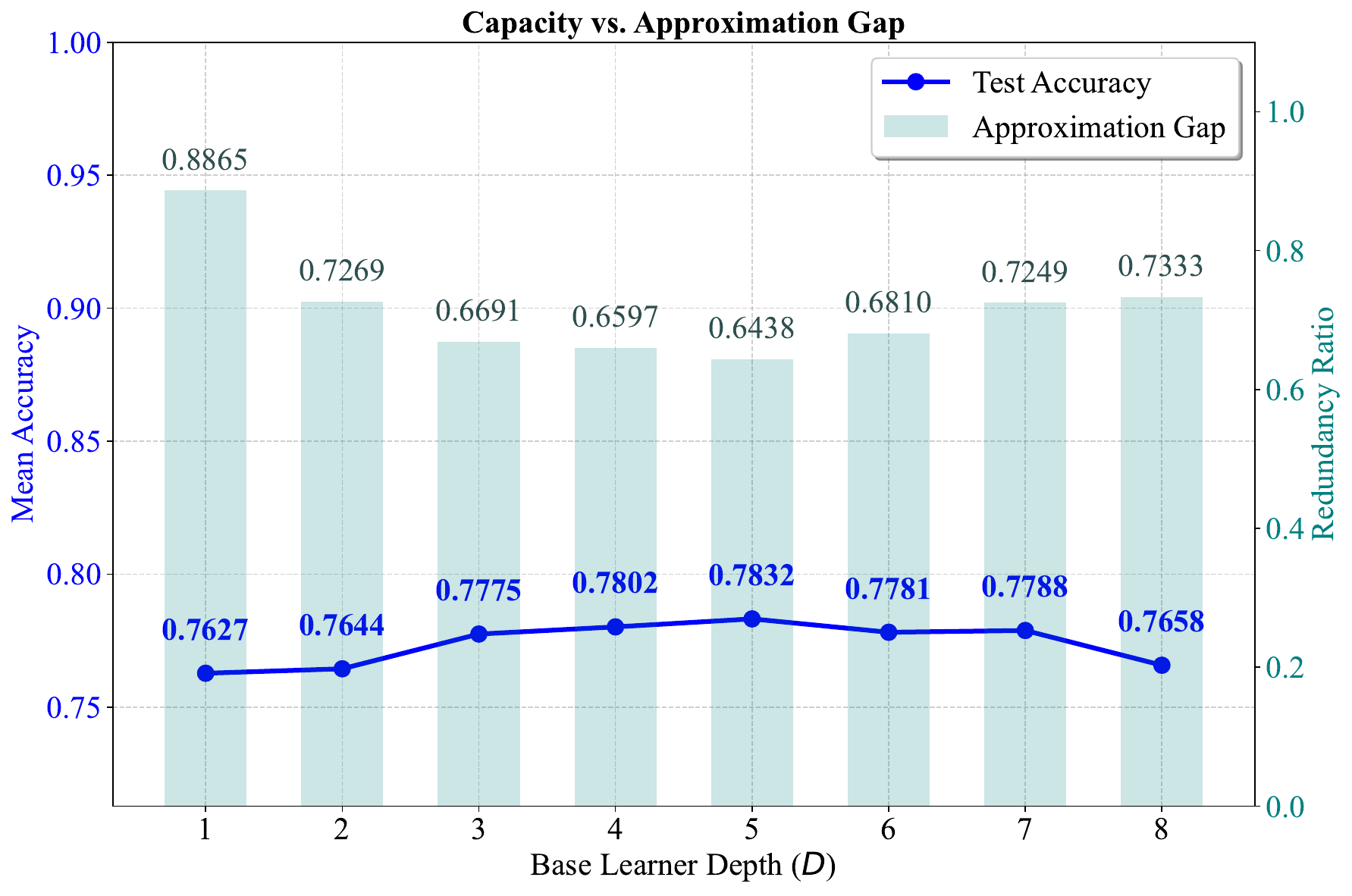}}
\caption{Effect of base learner depth on SCBoost on the Jasmine dataset. We reported test accuracy and redundancy ratio for tree depths $D\in[1,8]$ with $\alpha=0.9$.}
\label{fig:Capacity_vs_Approximation_Gap}
\end{center}
\vskip -0.2in
\end{figure}

\section{Time Complexity Analysis of SCBoost}
\label{sec:appendix_complexity}

Understanding the computational characteristics of SCBoost is essential for assessing its scalability. The time complexity is composed of three main components across $T$ iterations:
\begin{itemize}
\item \textbf{Weak Learner Training:} Training a regression tree of depth $D$ at each iteration takes $O(nd \log n)$, where $d$ is the number of features. Over $T$ iterations, this sums to $O(Tnd \log n)$.
\item \textbf{SRP:} At iteration $t$, SRP performs a truncated SVD on the prediction history matrix $\mathbf{H}^{(t-1)} \in \mathbb{R}^{n \times (t-1)}$. Note that while our implementation stores predictions instead of residuals, the dimensions remain identical to the residual history case. Thus, the computational cost remains $O(n(t-1)^2)$ per iteration. The cumulative cost over $T$ iterations is 
\begin{equation}
    \sum_{t=1}^{T} O(nt^2) = n \sum_{t=1}^{T} t^2 \approx O(nT^3).
\end{equation}
\item \textbf{CRW:} Solving the convex optimization problem for the final weights involves the covariance matrix of size $T \times T$, requiring $O(nT^2 + T^3)$.
\end{itemize}

Combining these components, the overall time complexity of SCBoost is dominated by the SRP step:
\begin{equation}
    O(nT^3 + Tnd \log n).
    \label{eq:complexity}
\end{equation}

To evaluate practical implications, Table \ref{tab:training_time_centered} shows the wall-clock training time of SCBoost against several baselines on all five datasets, using the same computational environment. The results revealed a critical limitation of our current SCBoost implementation. Compared to highly optimized frameworks like LGBM and XGB, SCBoost was consistently one to two orders of magnitude slower. For instance, on the OVA\_Uterus dataset, SCBoost took approximately 38.3 seconds, whereas LGBM and XGB finished in just 3.3 and 5.4 seconds, respectively. This significant computational overhead stems primarily from the SRP step, which performs an SVD on an ever-growing residual history matrix R at each iteration. This $O(n(t-1)^2)$ cost per iteration quickly becomes a bottleneck, especially for large $n$ or $T$.

\begin{table}
  \caption{The wall-clock time comparison in seconds across five benchmark datasets. The values represent the total time required for training each model.}
  \label{tab:training_time_centered}
  \begin{center}
    \begin{small}
      \begin{sc}
        \begin{tabular}{ccccccc}
          \toprule
          \textbf{Algorithm} & \textbf{OVA\_Uterus} & \textbf{Madelon} & \textbf{Jasmine} & \textbf{Bioresponse} & \textbf{Creditcard} \\
          \midrule
          RF                 & 1.7351               & 4.2346           & 1.2916           & 2.8999               & 1.1460              \\
          ADA                & 33.3357              & 8.2298           & 1.3550           & 13.8556              & 1.0961              \\
          GBDT               & 82.8157              & 20.6285          & 1.5783           & 27.7559              & 1.3447              \\
          XGB                & 5.3688               & 1.2161           & 0.1807           & 1.8623               & 0.1249              \\
          LGBM               & 3.2822               & 0.5968           & 0.1420           & 1.2505               & 0.1064              \\
          CAT                & 53.9535              & 8.2213           & 2.7265           & 34.7516              & 0.5754              \\
          NGB                & 78.6066              & 24.0307          & 3.6228           & 41.0628              & 1.7493              \\
          SCBoost            & 38.3050              & 29.4093          & 26.4415          & 29.1272              & 29.0681             \\
          \bottomrule
        \end{tabular}
      \end{sc}
    \end{small}
  \end{center}
  \vskip -0.1in
\end{table}

\textbf{Analysis of the Bottleneck.} Unlike traditional boosting algorithms (e.g., XGB, LGBM) which scale linearly with $T$ ($O(Tnd \log n)$), SCBoost exhibits a cubic dependence on the number of iterations ($T^3$) and linear dependence on sample size ($n$) for the SVD operations. This theoretical overhead explains the wall-clock time gap observed in Table \ref{tab:training_time_centered}. While this cost is significant, it is the price paid for the strict residual orthogonalization that eliminates redundancy. The memory complexity also grows as $O(nT)$, as the algorithm must store the full history of prediction vectors to maintain orthogonality. Consequently, the current exact implementation of SCBoost is best suited for scenarios where model compactness and diversity are prioritized over training speed, or where $n$ and $T$ are moderate.

\section{Parameter Configurations}
\label{sec:appendix_parac}
To ensure a fair and reproducible comparison of the algorithms’ intrinsic capabilities, we fixed only the random seed (random\_state=42) and n\_estimators=100 for reproducibility, while keeping all other parameters at their official recommended values as provided by the respective software packages (e.g., Scikit-learn). The proposed SCBoost algorithm was configured in the same way, following an official setup consistent with these baselines. It is well known that leading ensemble methods such as XGB and LGBM are highly sensitive to hyperparameter tuning, and their performance can often be substantially improved through data-specific optimization. Our deliberate choice to use default parameters is intended to assess the baseline robustness and practical usability of each method out of the box. This design isolates the intrinsic performance of the underlying architectures without the confounding effects of extensive, and often computationally expensive tuning, an important consideration for real-world applications where such tuning may be infeasible. For the high-dimensional OVA\_Uterus dataset, however, NGB’s official configuration (natural\_gradient=True) consistently failed with a LinAlgError. This issue, known to arise from the inversion of the Fisher Information Matrix in high-dimensional settings, persisted even with official stabilization options. To obtain a valid comparison, we disabled this mechanism by setting natural\_gradient=False, effectively reverting the model to standard gradient boosting. Results obtained under this modification are denoted with an asterisk (*) in our tables.

To ensure a fair and rigorous comparison against the performance ceiling of each baseline model, we conducted an extensive hyperparameter search. This process was automated using the \textbf{Optuna} framework \cite{akiba2019optuna}, a SOTA Bayesian optimization library. For each baseline algorithm, the optimization objective was to \textbf{maximize the mean accuracy score} obtained through a \textbf{10-fold stratified cross-validation} procedure on each dataset, with a fixed random seed (`rando\_state=42`) for reproducibility. The optimization for each model on each dataset was performed for \textbf{50 trials}, where each trial corresponded to a unique set of hyperparameters selected by Optuna's Tree-structured Parzen Estimator sampler. The specific hyperparameter search spaces were configured as follows. For \textbf{RF}, \texttt{n\_estimators} was searched in the integer range [50, 200] and \texttt{max\_depth} in [3, 8]. For \textbf{GBDT}, \textbf{XGB}, \textbf{LGBM}, and \textbf{CAT}, we searched \texttt{n\_estimators} in [50, 200], \texttt{max\_depth} in [3, 8], and \texttt{learning\_rate} on a log-uniform scale between $10^{-3}$ and 1.0. For \textbf{ADA} and \textbf{NGB}, we searched \texttt{n\_estimators} in [50, 200] and \texttt{learning\_rate} on a log-uniform scale between $10^{-3}$ and 1.0. The set of hyperparameters yielding the highest cross-validation accuracy was then used to report the final tuned performance in main Table 2. As stated in the main text, SCBoost was not subjected to this tuning process and retained its default parameters for all comparisons.

\section{Related Work}

\textbf{Boosting.} The history of boosting algorithms is a testament to the power of iterative improvement. From the foundational AdaBoost \cite{freund1997decision} and Gradient Boosting Machines \cite{friedman2001greedy} to modern titans like XGBoost \cite{chen2016xgboost} and LightGBM \cite{ke2017lightgbm}, the dominant trajectory of innovation has been overwhelmingly focused on two axes: computational acceleration (e.g., histogram-based binning, optimized tree construction) and specialized feature handling (e.g., CatBoost's target-based encoding) \cite{prokhorenkova2018catboost}. While these engineering breakthroughs have made boosting scalable to massive datasets, and recent advancements like ITBoost \cite{su2026itboost} have introduced information-theoretic approaches to enhance algorithmic robustness, they largely operate within the unchanged paradigm of sequential residual fitting. Consequently, the fundamental issue of learner redundancy has been largely sidestepped.

\textbf{Ensemble diversity.} The quest for ensemble diversity has yielded two main approaches. Randomization strategies \cite{geurts2006extremely}, effective in parallel paradigms like bagging \cite{breiman1996bagging, buhlmann2002analyzing}, are merely passive heuristics in boosting, failing to counteract its sequentially-induced redundancy. More active methods, like NCL \cite{chen2009regularized, chen2012ensemble, yu2019selective} and decorrelation penalties \cite{xie2017diverse, gu2018regularizing, zhang2025robust}, introduce diversity as a soft constraint within the loss function. This creates a contentious trade-off between accuracy and diversity, and can weaken individual learners or complicate convergence. Crucially, all these methods treat diversity as an external property to be encouraged, rather than an intrinsic goal of the learning target itself. They try to influence the learners, but they do not modify the problem each learner is asked to solve.

In summary, there exists a clear gap in the boosting paradigm: the lack of a principled, explicit, and direct mechanism that is architected specifically to eliminate redundancy during the sequential training process, without resorting to implicit loss-level penalties or passive randomization. The limitations of existing diversity-promoting methods when applied to boosting underscore the necessity for a novel approach that operates on a different level of abstraction. This critical gap motivates our work to fundamentally rethink how diversity is instilled in boosted ensembles.

\end{document}